\documentclass[11pt]{article}
\pdfoutput=1
\usepackage{enumerate}
\usepackage{pdfsync}
\usepackage[OT1]{fontenc}

\usepackage[usenames]{color}
\usepackage{graphicx}
\usepackage{caption}
\usepackage{smile}
\usepackage[colorlinks,
            linkcolor=red,
            anchorcolor=blue,
            citecolor=blue
            ]{hyperref}
\usepackage{fullpage}
\usepackage{hyperref}
\usepackage[protrusion=true,expansion=true]{microtype}
\usepackage{pbox}
\usepackage{setspace}
\usepackage{tabularx}
\usepackage{float}
\usepackage{wrapfig,lipsum}
\usepackage{microtype}
\usepackage{graphicx}
\usepackage{subfigure}
\usepackage{enumerate}
\usepackage{enumitem}
\usepackage{pgfplots}
\usetikzlibrary{arrows,shapes,snakes,automata,backgrounds,petri}
\usepackage{booktabs} 

\allowdisplaybreaks
\newcommand{\la}{\langle}
\newcommand{\ra}{\rangle}

\newcommand{\algname}{\text{BatchNeuralUCB}}
\newcommand{\gdname}{\text{TrainNN}}

\ifdefined\final
\usepackage[disable]{todonotes}
\else
\usepackage[textsize=tiny]{todonotes}
\fi
\newcommand{\alglinelabel}{%
  \addtocounter{ALC@line}{-1}
  \refstepcounter{ALC@line}
  \label
}

\makeatletter
\newcommand*{\rom}[1]{\expandafter\@slowromancap\romannumeral #1@}
\makeatother
\title{\huge Batched Neural Bandits} 

\author
{
	Quanquan Gu\thanks{Department of Computer Science, University of California, Los Angeles, CA 90095, USA; e-mail: {\tt qgu@cs.ucla.edu}} 
	~
	Amin Karbasi\thanks{Electrical Engineering, Computer Science, and  Statistics,  Yale University; e-mail: {\tt amin.karbasi@yale.edu}} 
	~
	Khashayar Khosravi\thanks{Google Research NYC; e-mail: {\tt khosravi@google.com}} 
	~
	Vahab Mirrokni\thanks{Google Research NYC; e-mail: {\tt mirrokni@google.com}} 
	~
	Dongruo Zhou\thanks{Department of Computer Science, University of California, Los Angeles, CA 90095, USA; e-mail: {\tt drzhou@cs.ucla.edu}} 
}

\begin{document}
\date{}
\maketitle
\begin{abstract}
In many sequential decision-making problems, the individuals are split into several batches and the decision-maker is only allowed to change her policy at the end of batches. These batch problems have a large number of applications, ranging from clinical trials to crowdsourcing. Motivated by this, we study the stochastic contextual bandit problem for general reward distributions under the batched setting. We propose the BatchNeuralUCB algorithm which combines neural networks with optimism to address the exploration-exploitation tradeoff while keeping the total number of batches limited. We study BatchNeuralUCB under both fixed and adaptive batch size settings and prove that it achieves the same regret as the fully sequential version while reducing the number of policy updates considerably. We confirm our theoretical results via simulations on both synthetic and real-world datasets.
\end{abstract}
\section{Introduction}\label{sec:intro}

In the stochastic contextual bandit problem, a learner sequentially picks actions over $T$ rounds (the \emph{horizon}). At each round, the learner observes $K$ actions, each associated with a $d$-dimensional feature vector. After selecting an action, she receives stochastic reward. Her goal is to maximize the cumulative reward attained over the horizon. Contextual bandits problems have been extensively studied in the literature \citep{langford2007epoch,  bubeck12regret,lattimore20bandit} and have a vast number of applications such as personalized news recommendation \citep{li2010contextual} and healthcare (see \citet{bouneffouf2019survey} and references therein).

Various reward models have been considered in the literature, such as linear models \citep{auer2002using,dani2008stochastic,li2010contextual, chu2011contextual,    abbasi2011improved, agrawal2013thompson}, generalized linear models \citep{filippi2010parametric, li2017provably}, and kernel-based models \citep{srinivas2009gaussian, valko2013finite}. Recently, neural network models that allow for a more powerful approximation of the underlying reward functions have been proposed \citep{riquelme2018deep, zhou2020neural, zhang2020neural,xu2020neural}. For example, the NeuralUCB algorithm \citep{zhou2020neural} can achieve near-optimal regret bounds while only requiring a very mild boundedness assumption on the rewards. However, a major shortcoming of NeuralUCB is that it requires updating the neural network parameters in every round, as well as optimizing the loss function over observed rewards and contexts. This task is computationally expensive and makes NeuralUCB considerably slow for large-scale applications and inadequate for the practical situations where the data arrives at a fast rate \citep{chapelle2011empirical}. 

In addition to the computational issues, many real-world applications require limited adaptivity, which allows the decision-maker to update the policy at only certain time steps. Examples include multi-stage clinical trials \citep{perchet2016batched}, crowdsourcing platforms \citep{kittur2008crowdsourcing}, and running time-consuming simulations for reinforcement learning \citep{le2019batch}. This restriction formally motivates the \emph{batched} multi-armed bandit problem that was studied first in \citet{perchet2016batched} for the two-armed bandit with noncontextual rewards. Their results have been recently extended to many-armed setting \citep{gao2019batched,esfandiari2019batched,jin2020double}, linear bandits \citep{esfandiari2019batched}, and also linear contextual bandits \citep{han2020sequential, ruan2020linear}. A closely related literature of \emph{rarely switching} multi-armed bandit problem measures the limited adaptivity by the number of policy switches \citep{cesa2013online, dekel2014bandits, simchi2019phase, ruan2020linear}, where in contrast to the batched models that the policy updates are at pre-fixed time periods, the policy updates are adaptive and can depend on the previous context and reward observations. While these papers provide a complete characterization of the optimal number of policy switches in both cases for the stochastic contextual bandit problem with the linear reward, the extension of results to more general rewards remains unstudied.

In this paper, we propose a \algname~algorithm that uses neural networks for estimating rewards while keeping the total number of policy updates to be small. \algname~addresses both limitations described above: (1) it reduces the computational complexity of NeuralUCB, allowing its usage in large-scale applications, and (2) it limits the number of policy updates, making it an excellent choice for settings that require limited adaptivity. It is worth noting that while the idea of limiting the number of updates for neural networks has been used in \citet{riquelme2018deep,xu2020neural} and the experiments of \citet{zhou2020neural}, no formal results on the number of batches required or the optimal batch selection scheme are provided. Our main contributions can be summarized as follows. 

\begin{itemize}
    \item We propose $\algname$ which, in sharp contrast to NeuralUCB, only updates its network parameters at most $B$ times, where $B$ is the number of batches. We propose two update schemes: the \emph{fixed} batch scheme where the batch grid is pre-fixed, and the \emph{adaptive} batch scheme where the selection of batch grid can depend on previous contexts and observed rewards. When $B = T$, $\algname$ degenerates to NeuralUCB. 
    \item We prove that for $\algname$ with fixed batch scheme, the regret is bounded by $\tilde O(\tilde d\sqrt{T} + \tilde dT/B)$, where $\tilde d$ is the effective dimension (See Definition \ref{def:effective_dimension}). For adaptive batch scheme, for any choice of $q$, the regret is bounded by $\tilde O(\sqrt{\max\{q, (1+TK)^{\tilde d}/q^B\}}\tilde d\sqrt{T})$, where $q$ is the parameter that determines the adaptivity of our algorithm (See Algorithm \ref{algorithm:3} for details), and $K$ is the number of arms. Therefore, to obtain the same regret as its fully sequential counterpart, \algname~only requires $\Omega(\sqrt{T})$ for fixed and $\Omega(\log T)$ for adaptive batch schemes. These bounds match the lower bounds presented in the batched linear bandits \citep{han2020sequential} and rarely switching linear bandits \citep{ruan2020linear} respectively. 

    \item We carry out numerical experiments over synthetic and real datasets to confirm our theoretical findings. In particular, these experiments demonstrate that in most configurations with fixed and adaptive schemes, the regret of the proposed $\algname$ remains close to the regret of the fully sequential NeuralUCB algorithm, while the number of policy updates as well as the running time are reduced by an order of magnitude.
    
\end{itemize}

\textbf{Notations} We use lower case letters to denote scalars,  lower and upper case bold letters to denote vectors and matrices. We use $\| \cdot \|$ to indicate Euclidean norm, and for a semi-positive definite matrix $\bSigma$ and any vector $\xb$, $\| \xb \|_{\bSigma} := \| \bSigma^{1/2} \xb \| = \sqrt{\xb^{\top} \bSigma \xb}$. 
We also use the standard $O$ and $\Omega$ notations. We say $a_n = O(b_n)$ if and only if $\exists C > 0, N > 0, \forall n > N, a_n \le C b_n$; $a_n = \Omega(b_n)$ if $a_n \ge C b_n$. The notation $\tilde{O}$ is used to hide logarithmic factors. Finally, we use the shorthand that $[n]$ to denote the set of integers $\{1, . . . , n\}$.

\section{Related Work}\label{sec:related}
The literature on the contextual multi-armed problem is vast. Due to the space limitations, we only review the existing work on batched bandits and bandits with function approximations here and refer the interested reader to recent monographs by \citet{slivkins2019introduction} and \citet{lattimore20bandit} for a thorough overview. 

\paragraph{Batched Bandits.} The design of batched multi-armed bandit models can be traced back to UCB2 \citep{auer02finite} and Improved-UCB \citep{auer2010ucb} algorithms originally for the fully sequential setting. \citet{perchet2016batched} provided the first systematic analysis of the batched stochastic multi-armed bandit problem and established near-optimal gap-dependent and gap-independent regrets for the case of two arms ($K=2$). \citet{gao2019batched} extended this analysis to the general setting of $K>2$. They proved regret bounds for both adaptive and non-adaptive grids. \citet{esfandiari2019batched} improved the gap-dependent regret bound for the stochastic case and provided lower and upper bound regret guarantees for the adversarial case. They also establish regret bounds for the batched stochastic linear bandits. 

Our work in the batched setting is mostly related to \citet{han2020sequential, ruan2020linear}. In particular, \citet{han2020sequential} studied the batched stochastic linear contextual bandit problem for both adversarial and stochastic contexts. For the case of adversarial contexts, they show that the number of batches $B$ should be at least $\Omega(\sqrt{dT})$. 
\citet{ruan2020linear} studied the batched contextual bandit problem using distributional optimal designs and extended the result of \citep{han2020sequential}. They also studied the minimum adaptivity needed for the rarely switching contextual bandit problems in both adversarial and stochastic context settings. In particular, for adversarial contexts, they proved a lower bound of $\Omega((d \log T)/\log(d \log T))$. Our work, however, is different from \citet{han2020sequential, ruan2020linear} as we do not require any assumption on the linearity of the reward functions; similar to NeuralUCB \citep{zhou2020neural}, our regret analysis only requires the rewards to be bounded.


\paragraph{Bandits with Function Approximation.} Given the fact that Deep Neural Network (DNN) models enable the learner to make use of nonlinear models with less domain knowledge, \citet{riquelme2018deep,zahavy2019deep} studied \emph{neural-linear bandits}. 
In particular, they used all but the last layers of a DNN as a feature map, which transforms contexts from the raw input space to a low-dimensional space, usually with better representation and less frequent updates. Then they learned a linear exploration policy on top of the last hidden layer of the DNN with more frequent updates. Even though these attempts have achieved great empirical success, they do not provide any regret guarantees. \citet{zhou2020neural} proposed NeuralUCB algorithm that uses neural networks to estimate reward functions while addressing the exploration-exploitation tradeoff using the upper confidence bound technique.  \citet{zhang2020neural} extended their analysis to Thompson Sampling. \citet{xu2020neural} proposed Neural-LinUCB which shares the same spirit as \emph{neural-linear bandits} and proved $\tilde O(\sqrt{T})$ regret bound.

\section{Problem Setting}\label{sec:setting}

In this section, we present the technical details of our model and our problem setting. 

\paragraph{Model.} We consider the stochastic $K$-armed contextual bandit problem, where the total number of rounds $T$ is known. At round $t \in [T]$, the learner observes the context consisting of $K$ feature vectors: $\{\xb_{t,a} \in \RR^d~|~a \in [K]\}$. For brevity, we denote the collection of all contexts $\{\xb_{1,1}, \xb_{1,2},\dots,\xb_{T,K}\}$ by $\{\xb^i\}_{i=1}^{TK}$. 

\vspace{-5pt}

\paragraph{Reward.} Upon selecting action $a_t$,
she receives a stochastic reward $r_{t,a_t}$. In this work, we make the following assumption about reward generation: for any round $t$,
\begin{align}
    r_{t,a_t} = h(\xb_{t,a_t}) + \xi_t,\label{eq:reward function}
\end{align}
where $h$ is an unknown function satisfying $0 \leq h(\xb) \leq 1$ for any $\xb$, and $\xi_t$ is $\nu$-sub-Gaussian noise conditioned on $\xb_{1,a_1},\dots,\xb_{t-1,a_{t-1}}$ satisfying $\EE [\xi_t|\xb_{1,a_1},\dots,\xb_{t-1,a_{t-1}}] = 0$.

\vspace{-5pt}

\paragraph{Goal.} The learner wishes to maximize the following \emph{pseudo regret} (or \emph{regret} for short):
\begin{align}
    R_T = \EE\bigg[\sum_{t=1}^T (r_{t,a_t^*} - r_{t,a_t})\bigg],\label{def:regret}
\end{align}
where $a_t^* = \argmax_{a \in[K]}\EE [r_{t,a}]$ is the optimal action at round $t$ that maximizes the expected reward. 

\vspace{-5pt}

\paragraph{Reward Estimation.} In order to learn the reward function $h$ in Eq. \eqref{eq:reward function}, we propose to use a fully connected neural networks with depth $L \geq 2$:
\begin{align}
    &f(\xb; \btheta) = \sqrt{m}\Wb_L \sigma\Big(\Wb_{L-1}\sigma\big(\cdots \sigma (\Wb_1\xb)\big)\Big), \label{def:network}
\end{align}
where $\sigma(x) = \max\{x,0\}$ is the rectified linear unit (ReLU) activation function, $\Wb_1 \in \RR^{m \times d}, \Wb_i \in \RR^{m \times m}, 2 \leq i \leq L-1, \Wb_L \in \RR^{m \times 1}$, and $\btheta = [\text{vec}(\Wb_1)^\top,\dots,\text{vec}(\Wb_L)^\top]^\top \in \RR^{p}$ with $p = m+md+m^2 (L-1)$. Without loss of generality, we assume that the width of each hidden layer is the same (i.e., $m$) for convenience in the analysis. We denote the gradient of the neural network function by $\gb(\xb; \btheta) = \nabla_{\btheta} f(\xb; \btheta) \in \RR^p$. 

\vspace{-5pt}

\paragraph{Batch Setting.} In this work, we consider the \emph{batch bandits} setting in which the entire horizon of $T$ rounds is divided into $B$ batches. Formally, we define a grid $\mathcal{T} = \{t_0, t_1, \cdots, t_B \}$, where $1 = t_0 < t_1 < t_2 < \cdots < t_B = T+1$ are the start and end rounds of the batches. Here, the interval $[t_{b-1}, t_b)$ indicates the rounds belonging to batch $b \in [B]$. The learner selects her policy at the beginning of each batch and executes it for the entire batch. She observes all collected rewards at the end of this batch and then updates her policy for the next batch. The batch model consists of two specific schemes. In the \emph{fixed batch size scheme}, the points in the grid $\mathcal{T}$ are pre-fixed and cannot be altered during the execution of the algorithm. In the \emph{adaptive batch size scheme}, however, the beginning and the end rounds of each batch are decided dynamically by the algorithm. 

\section{Algorithms}\label{sec:algorithms}


We propose our algorithm $\algname$ in this section. In essence, $\algname$ uses a neural network $f(\xb; \btheta)$ to predict the reward of the context $\xb$ and upper confidence bounds computed from the network to guide the exploration \citep{auer2002using}, which is similar to NeuralUCB \citep{zhou2020neural}. The main difference is that $\algname$ does not update its parameter $\btheta$ at each round. Instead, $\algname$ specifies either a fixed or adaptive batch grid $\mathcal{T}= \{t_0, t_1,\dots, t_B\}$. At the beginning of the $b$-th batch, the algorithm updates the parameter $\btheta$ of the neural network to $\btheta_b$ by optimizing a regularized square loss trained on all observed contexts and rewards using gradient descent. The training procedure is described in Algorithm \ref{algorithm:GD}. 
Meanwhile, within each batch, $\algname$ maintains the covariance matrix $\Zb_{t_b}$ which is calculated over the gradients of the observed contexts, each taken with respect to the estimated parameter of the neural network at the beginning of that contexts' corresponding batch. Based on $\btheta_b$ and $\Zb_{t_b}$, $\algname$ calculates the UCB estimate of reward  $f_b(\cdot)$, as Line \ref{alg:line1} in Algorithm \ref{algorithm:3} suggests. The function $f_b(\cdot)$ is used to select actions during the $b$-th batch. In particular, at round $t$, $\algname$ receives contexts $\{\xb_{t,a}\}_{a=1}^K$ and picks $a_t$ which maximizes the optimistic reward $f_b(x_{t,a})$ (see Line \ref{alg:actionselection}). Once this batch finishes, the rewards $r_{t, a_t}$ collected during this batch are observed (Line \ref{alg:reward}), and the process continues.

\subsection{Fixed Batch Size Scheme}\label{sec:fixed-batch}
For the fixed batch scheme, $\algname$ predefines the batch grid $\mathcal{T} = \{t_0, t_1, \dots\, t_B \}$ as a deterministic set depending on the time horizon $T$ and number of batches $B$.

\begin{algorithm}[ht]
	\caption{$\algname$}\label{algorithm:3}
	\begin{algorithmic}[1]
	\REQUIRE A neural network $f(\xb; \btheta)$ initialized with parameter $\btheta_0$, batch number $B$, ratio parameter $q$ (only needed for adaptive batch size), regularization parameter $\lambda$, step size $\eta$, number of gradient descent steps $J$
	\STATE $\Zb_1 = \lambda \Ib$, $b = 0$, $t_0 = 1$
    \FOR{$t = 1,\dots,T$}
    \IF {
    \begin{align}
    &\textbf{Fixed Batch Size Scheme: }t = b\cdot \lfloor T/B\rfloor +1\alglinelabel{alg:fix} \\
        &\textbf{Adaptive Batch Size Scheme: }\det(\Zb_t)>q\cdot \det(\Zb_{t_b})\text{ and }b \leq B-2, \alglinelabel{alg:ada}
    \end{align}
    }
    \STATE $b \leftarrow b+1$, $t_{b} \leftarrow t$
    \STATE Observe rewards $\{r_{i, a_i}\}_{i=t_{b-1}}^{t_{b}-1}$ corresponding to contexts $\{\xb_{i, a_i}\}_{i=t_{b-1}}^{t_{b}-1}$ \alglinelabel{alg:reward}
    \STATE $\btheta_b \leftarrow \text{TrainNN} (\lambda, \eta, J, m,\{\xb_{i, a_i}\}_{i=1}^{t_b - 1},$ \\$ \{r_{i, a_i}\}_{i=1}^{t_b - 1}, \btheta_0)$
    \STATE 
    $f_b(\cdot)\leftarrow f(\cdot; \btheta_b) + \beta_{t_b}\sqrt{\gb(\cdot; \btheta_b)\Zb_{t_b}^{-1}\gb(\cdot; \btheta_b)/m}$, \alglinelabel{alg:line1}
    \ENDIF 
    \STATE Receive $\{\xb_{t,a}\}_{a = 1}^K$
    \STATE Select $a_t\leftarrow \argmax_{a \in [K]}  f_b(\xb_{t,a})$ \alglinelabel{alg:actionselection}
    \STATE Set $\Zb_{t+1} \leftarrow \Zb_t + \gb(\xb_{t,a_t}; \btheta_b)^\top\gb(\xb_{t,a_t}; \btheta_b)/m$\alglinelabel{alg:line2}
        \ENDFOR
        \STATE $t_{b+1} = T+1$
	\end{algorithmic}
\end{algorithm}

\begin{algorithm}[th]
	\caption{$\gdname$} \label{algorithm:GD} 
	\begin{algorithmic}[1]
	\REQUIRE Regularization parameter $\lambda$, step size $\eta$, number of gradient descent steps $J$, network width $m$, actions $\{\xb_t\}$, rewards $\{r_t\}$, initial parameter $\btheta^{(0)}$.
    \STATE Define $\cL (\btheta) = \sum_{i=1}^t ( f(\xb_i; \btheta) - r_i)^2/2 + m\lambda\|\btheta - \btheta^{(0)}\|_2^2/2$.
    \FOR{$j = 0, \dots, J-1$}
    \STATE $\btheta^{(j+1)} \leftarrow \btheta^{(j)} - \eta\nabla \cL(\btheta^{(j)})$
    \ENDFOR 
    \ENSURE $\btheta^{(J)}$.
	\end{algorithmic}
\end{algorithm}

In particular, $\algname$ selects the simple uniform batch grid, with $t_b = b\cdot \lfloor T/B\rfloor +1$, as suggested in Eq. \eqref{alg:fix}. It is easy to see that when $B = T$, $\algname$ updates the network parameters at each round, reducing to NeuralUCB. \citep{han2020sequential} also studied the fixed batch size scheme, but for the linear reward.

\subsection{Adaptive Batch Size Scheme}\label{sec:adaptive-batch}
Unlike the fixed batch size scheme, in the adaptive batch size scheme, $\algname$ does not predefine the batch grid. Instead, it dynamically selects the batch grids based on the previous observations. Specifically, at any time $t$, the algorithm calculates the determinant of the covariance matrix and keeps track of its ratio to the determinant of the covariance matrix calculated at the end of the previous batch. If this ratio is larger than a hyperparameter $q$ and the number of utilized batches is less than the budget $B$, then $\algname$ starts a new batch. This idea used in the adaptive batch size scheme is similar to the \emph{rarely switching} updating rule introduced in \citet{abbasi2011improved} for linear bandits. The difference is that while \citet{abbasi2011improved} applies this idea directly to the contexts $\{\xb^i\}_i$, Algorithm \ref{algorithm:3} applies it to the gradient mapping of contexts.

\section{Main Results}\label{sec:theory}
In this section, we propose our main theoretical results about Algorithm~\ref{algorithm:3}. First, we need the definition of the neural tangent kernel (NTK) matrix \citep{jacot2018neural}. 
\begin{definition}\label{def:ntk}
Let $\{\xb^i\}_{i=1}^{TK}$ be a set of contexts.  Define
\begin{align*}
    \tilde \Hb_{i,j}^{(1)} &= \bSigma_{i,j}^{(1)} = \la \xb^i, \xb^j\ra, ~~~~~~~~  \Ab_{i,j}^{(l)} = 
    \begin{pmatrix}
    \bSigma_{i,i}^{(l)} & \bSigma_{i,j}^{(l)} \\
    \bSigma_{i,j}^{(l)} & \bSigma_{j,j}^{(l)} 
    \end{pmatrix},\notag \\
    \bSigma_{i,j}^{(l+1)} &= 2\EE_{(u, v)\sim N(\zero, \Ab_{i,j}^{(l)})} \left[\sigma(u)\sigma(v)\right],\notag \\
    \tilde \Hb_{i,j}^{(l+1)} &= 2\tilde \Hb_{i,j}^{(l)}\EE_{(u, v)\sim N(\zero, \Ab_{i,j}^{(l)})} \left[\sigma'(u)\sigma'(v)\right] + \bSigma_{i,j}^{(l+1)}.\notag
\end{align*}
Then, $\Hb = (\tilde \Hb^{(L)} + \bSigma^{(L)})/2$ is called the \emph{neural tangent kernel (NTK)} matrix on the context set $\{\xb^i\}_i$. For simplicity, let $\hb \in \RR^{TK}$ denote the vector $(h(\xb^i))_{i=1}^{TK}$. 
\end{definition}

We need the following assumption over the NTK gram matrix $\Hb$. 
\begin{assumption}\label{assumption:input}
The NTK matrix satisfies $\Hb \succeq \lambda_0\Ib$. 
\end{assumption}
\begin{remark}
Assumption \ref{assumption:input} suggests that the NTK matrix $\Hb$ is non-singular. Such a requirement can be guaranteed as long as \emph{no} two contexts in $\{\xb^i\}_i$ are parallel \citep{du2018gradient}.
\end{remark}
 We also need the following assumption over the initialized parameter $\btheta_0$ and the contexts $\xb^i$. 
\begin{assumption}\label{assumption:context}
For any $1 \leq i \leq TK$, the context $\xb^i$ satisfies $\|\xb^i\|_2 = 1$ and $[\xb^i]_j =[\xb^i]_{j+d/2}$. Meanwhile, the initial parameter $\btheta_0 = [\text{vec}(\Wb_1)^\top,\dots,\text{vec}(\Wb_L)^\top]^\top$ is initialized as follows: for $1 \leq l\leq L-1$, $\Wb_l$ is set to $\begin{pmatrix}
    \Wb & \zero \\
    \zero & \Wb
    \end{pmatrix}$,
where each entry of $\Wb$ is generated independently from $N(0, 4/m)$; $\Wb_L$ is set to $(\wb^\top, -\wb^\top)$, where each entry of $\wb$ is generated independently from $N(0, 2/m)$.
\end{assumption}
\begin{remark}
Assumption \ref{assumption:context} suggests that the context $\xb^i$ and the initial parameter $\btheta_0$ should be `symmetric' considering each coordinate. It can be verified that under such an assumption, for any $i \in [TK]$ we have $f(\xb^i; \btheta_0) = 0$, which is crucial to our analysis. Meanwhile, for any context $\xb$ that does not satisfy the assumption, we can always construct a satisfying new context $\xb'$ by setting $\xb' = [\xb^\top, \xb^\top]^\top/\sqrt{2}$.
\end{remark}
We also need the following definition of the effective dimension. 
\begin{definition}\label{def:effective_dimension}
The effective dimension $\tilde d$ of the neural tangent kernel matrix on contexts $\{\xb^i\}_{i=1}^{TK}$ is defined as
\begin{align}
    \tilde d = \frac{\log \det (\Ib + \Hb/\lambda)}{\log (1+TK/\lambda)}.\notag
\end{align}
\end{definition}
\begin{remark}
The notion of effective dimension $\tilde d$ is similar to the \emph{information gain} introduced in \citet{srinivas2009gaussian} and \emph{effective dimension} introduced in  \citet{valko2013finite}. Intuitively, $\tilde d$ measures how quickly the eigenvalues of $\Hb$ diminish, and it will be upper bounded by the dimension of the RKHS space spanned by $\Hb$ \citep{zhou2020neural}. 
\end{remark}

The following two theorems characterize the regret bounds of $\algname$ under two different update schemes. We first show the regret bound of $\algname$ under the fixed batch size update scheme. 
\begin{theorem}\label{thm:batch}
Suppose Assumptions \ref{assumption:input} and \ref{assumption:context} hold. Setting $m = \text{poly}(T,L,K,\lambda^{-1}, \lambda_0^{-1},  S^{-1}, \log(1/\delta))$ and $\lambda \geq  S^{-2}$, where $S$ is a parameter satisfying $S \geq \sqrt{2\hb^\top\Hb^{-1}\hb}$. There exist positive constants $C_1, C_2, C_3$ such that, if 
\begin{align}
    \beta_t &= C_1\bigg[\bigg(\nu\sqrt{\log \frac{\det\Zb_t}{\det \lambda \Ib} - 2\log \delta}+\sqrt{\lambda}S\bigg) \notag \\
    &+ (\lambda+tL)(1-\eta m\lambda)^{J/2}\sqrt{t/\lambda}\bigg],\notag
\end{align}
$J = 2\log(\lambda S/(\sqrt{T}\lambda + C_2 T^{3/2}L))TL/\lambda$, $\eta = C_3 (mTL+m\lambda)^{-1}$, then with probability at least $1-\delta$, the regret of Algorithm \ref{algorithm:3} with fixed batch size scheme is bounded as follows:
\begin{align}
    R_T= \tilde O\bigg(\Big(\nu\tilde d + \sqrt{\lambda \tilde d}S\Big)\sqrt{T} + \tilde d T/B\bigg).\notag
\end{align}
\end{theorem}

\begin{remark}\label{remark:1}
Suppose $h$ belongs to the RKHS space of NTK kernel $\cH$ with a finite RKHS norm $\|h\|_{\cH}$, then $\sqrt{\hb^\top\Hb^{-1}\hb} \leq \|h\|_{\cH}$ (Appendix A.2, \citealt{zhou2020neural}). Therefore, by treating $\nu$ as a constant and setting $S = \sqrt{2}\|h\|_{\cH}$, $\lambda  =  S^{-2}$, the regret is on the order $\tilde O(\tilde d\sqrt{T} + \tilde dT/B)$.
This suggests setting $B = \sqrt{T}$ in order to obtain the standard regret $\tilde O(\tilde d\sqrt{T})$. 
\end{remark}
\begin{remark}
\citet{han2020sequential} proposed a lower bound on the regret of $2$-armed linear bandits with $d$-dimensional contexts, which suggests that for any algorithm with a fixed $B$-batch size scheme, the regret is no less than
\begin{align}
    \Omega(\sqrt{dT} + \min\{T\sqrt{d}/B, T/\sqrt{B}\}).\label{lower}
\end{align}
Eq. \eqref{lower} shows that to obtain an $\tilde O(\sqrt{T})$-regret, at least $\Omega(\sqrt{T})$ number of batches are needed, which implies that our choice of $B$ as $\sqrt{T}$ is tight.  
\end{remark}
We have the following theorem for Algorithm \ref{algorithm:3} under the adaptive batch size scheme.

\begin{theorem}\label{thm:rare}
Suppose Assumptions \ref{assumption:input} and \ref{assumption:context} hold. Let $S,\lambda, J, \eta, \{\beta_t\}$ be selected as in Theorem \ref{thm:batch}. Then with probability at least $1-\delta$, the regret of Algorithm \ref{algorithm:3} with the adaptive batch size scheme can be bounded by
\begin{align}
    R_T 
    & = \tilde O\bigg(\sqrt{\max\{q, (1+TK/\lambda)^{\tilde d}/q^B\}}\Big(\nu\tilde d + \sqrt{\lambda\tilde d}S\Big)\sqrt{T}  \bigg).\notag
\end{align}

\end{theorem}

\begin{remark}
By treating $\nu$ as a constant and assuming that $h$ belongs to the RKHS space of NTK kernel $\cH$ with a finite RKHS norm $\|h\|_{\cH}$, and by setting $S$ and $\lambda$ as Remark \ref{remark:1} suggests, the regret is bounded by $\tilde O(\sqrt{\max\{q, (1+TK)^{\tilde d}/q^B\}}\tilde d\sqrt{T})$. 
\end{remark}
\begin{remark}
To achieve an $\tilde O(\tilde d\sqrt{T})$ regret, here $B$ needs to be chosen as $\Omega(\tilde d \log(1+TK/\lambda))$ and $q = \tilde\Theta((1+TK/\lambda)^{\tilde d/B})$. As a comparison, for the linear bandits case, \citet{ruan2020linear} has shown that an $O(d \log d \log T)$ number of batches is necessary to achieve a $\tilde O(\sqrt{T})$ regret. Therefore, our choice of $B$ as $\tilde d \log(T)$ is tight up to a $\log \tilde d$ factor.
\end{remark}

\section{Numerical Experiments}
\label{sec:sim}

\begin{figure*}[t]
    \centering
    \includegraphics[width=0.9\textwidth]{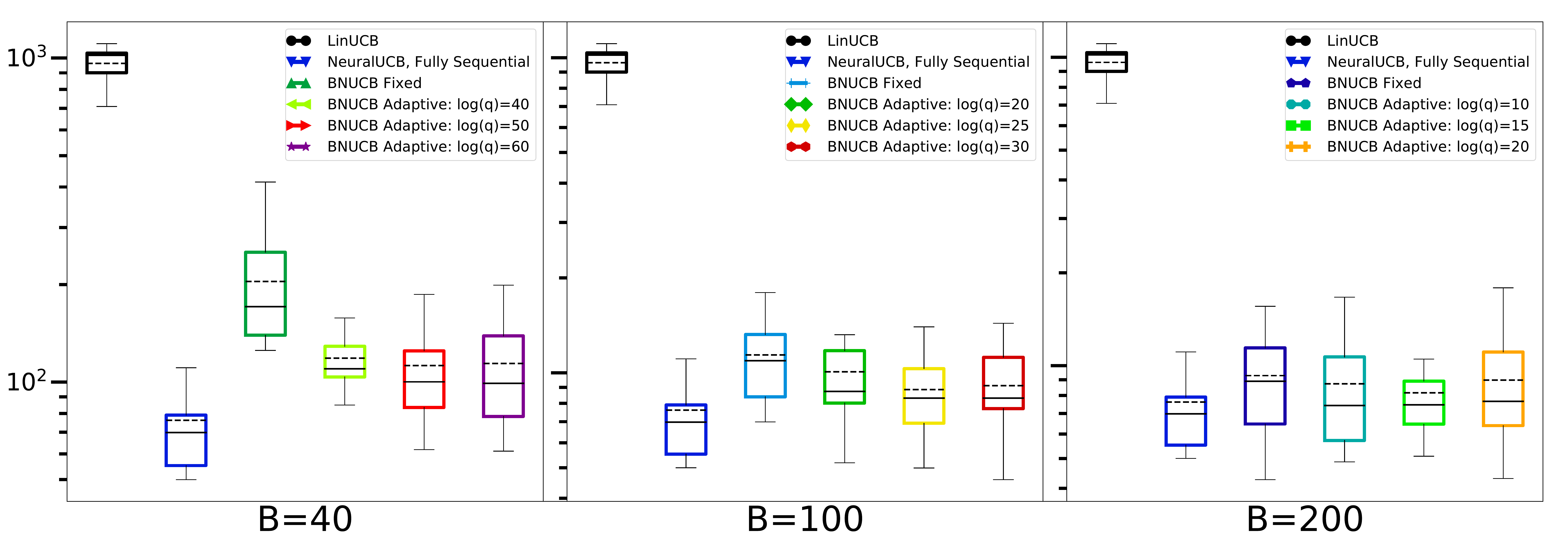}
    \caption{Distribution of per-instance regret on Synthetic data. The solid and dashed lines indicate the median and the mean respectively. Note that BNUCB stands for $\algname$ and also that the regrets are plotted on the log scale.} \label{fig:cosine}
\end{figure*}

In this section, we run numerical experiments to validate our theoretical findings. In what follows we consider both real and synthetically generated data. Due to space limitations, we defer the discussion of hyperparameter tuning of algorithms and also additional simulations to Appendix \ref{app:additional_sim}.

\subsection{Synthetic Data}
\label{sec:sim-synth}

We compare the performance of our proposed $\algname$ (BNUCB) algorithm with fixed and adaptive size batches for several values of $B$ and $q$ with two fully sequential benchmarks on synthetic data generated as follows. Let $T=2000, d=10, K=4$, and consider the cosine reward function given by $r_{t,a} = \cos(3 \xb_{t,a}^\top \btheta^*) + \xi_t$ where $\{\xb_{t,1}, \xb_{t,2}, \cdots, \xb_{t,K}\}$ are contexts generated at time $t$ according to $U[0, 1]^d$ independent of each other. The parameter $\btheta^*$ is the unknown parameter of model that is generated according to $U[0, 1]^d$, normalized to satisfy $\| \btheta^* \|_2 = 1$. The noise $\xi_t$ is generated according to $N(0, 0.25)$ independent of all other variables. 

The fully sequential benchmarks considered are: (1) NeuralUCB algorithm \citep{zhou2020neural} and (2) LinUCB algorithm \citep{li2010contextual}. For BatchNeuralUCB and NeuralUCB algorithms, we consider two-layer neural networks with $m=200$ hidden layers. We report this process for $10$ times and generate the following plots:

$\bullet$ Box plot of the total regret of algorithms together with its standard deviation.

$\bullet$ Scatter plot of the total regret vs execution time for $5$ simulations randomly selected out of all $10$ simulations. 

\vspace{-5.6pt}

\paragraph{Results.} The results are depicted in Figures \ref{fig:cosine} and \ref{fig:cosine-time}. We can make the following observations. First, the regret of LinUCB is almost $10$ times worse than the fully sequential NeuralUCB, which is potentially due to the model misspecification. Our proposed BNUCB works pretty well in both fixed and adaptive schemes, while keeping the total number of policy updates and also the running time small. In fact, for all models, except the fixed batched setting with $B=40$, the regret of BNUCB is within a factor of two of its fully sequential counterpart. At the same time, the number of policy updates and execution times of all configurations of BNUCB for all pairs of $(B,q)$ are almost ten times smaller than the fully sequential version. Second, for a given batch size $B$, the adaptive batch scheme configurations have better performance compared to the fixed ones. 

\subsection{Real Data}
\label{sec:sim-real}
We repeat the above simulation this time using the Mushroom dataset from the UCI repository.\footnote{This dataset can be found here \href{https://archive.ics.uci.edu/ml/datasets/mushroom}{https://archive.ics.uci.edu/ml/datasets/mushroom}} The dataset is originally designed for classifying edible vs poisonous mushrooms. It contains $n=8124$ samples each with $d=22$ features, each belonging to one of $K=2$ classes. For each sample $s_t = (c_t, l_t)$ with context $c_t \in \RR^d$ and label $l_t \in [K]$, we consider the zero-one reward defined as $r_{t, a_t} = \ind\{a_t = l_t\}$ and generate our context vectors as $\{\xb_{t,a} \in \RR^{Kd}: \xb_{t,a} = [\underbrace{0, \cdots, 0}_{a-1 \text{~times}}, c_i, \underbrace{0, \cdots, 0}_{K-a \text{~times}}], a \in [K]\}$. 

We compare the performance of algorithms on similar metrics as those described in Section~\ref{sec:sim-real} over $10$ random Monte Carlo simulations and report the results. For each instance, we select $T=2000$ random samples without replacement from the dataset and run all algorithms on that instance. Note that in this simulation, both NeuralUCB and $\algname$ use two-layer neural networks with $m=100$ hidden layers. 

\begin{figure*}[hptb]
    \centering
    \includegraphics[width=0.9\textwidth]{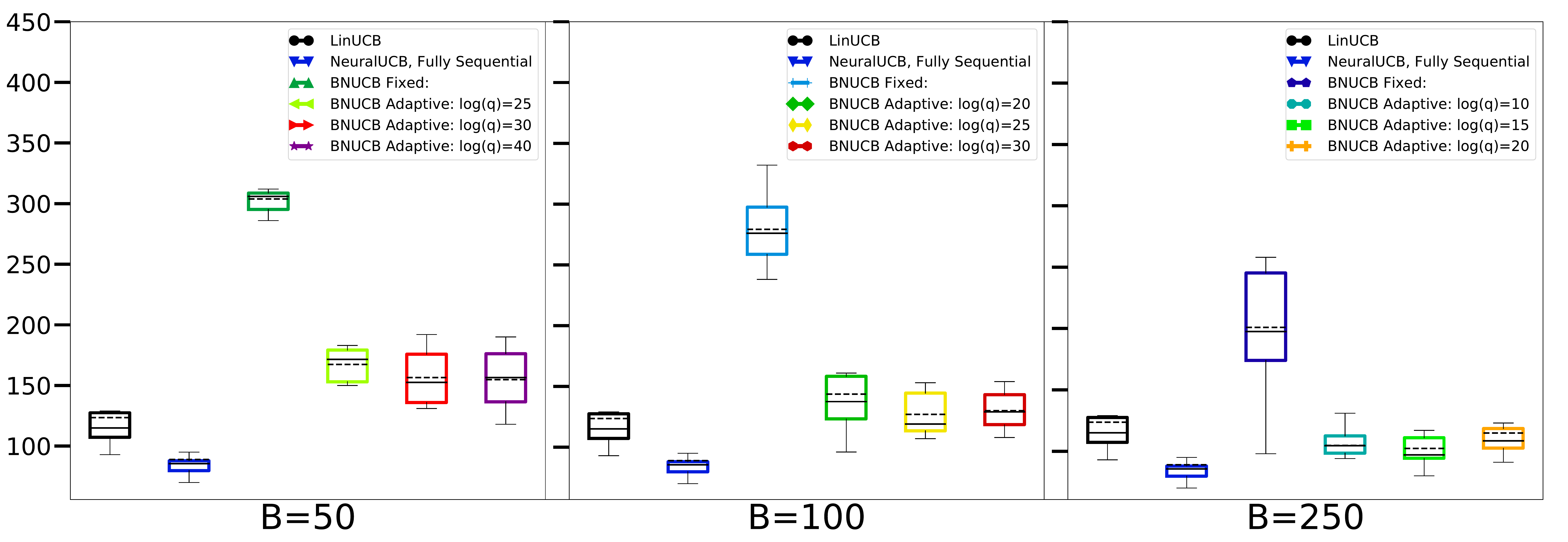}
    \caption{Distribution of per-instance regret on Mushroom dataset. The solid and dashed lines indicate the median and the mean respectively. Note that BNUCB stands for $\algname$.} \label{fig:mushroom}
\end{figure*}


\vspace{-7pt}

\paragraph{Results.} The results are depicted in Figures \ref{fig:mushroom} and \ref{fig:mushroom-time}. We can make the following observations. First, among fully sequential models, NeuralUCB outperforms LinUCB although it is $1000$ times slower. As the number of batches used in BNUCB increases the regret decreases and it gets closer to that of fully sequential NeuralUCB. For instance, in all adaptive batch scheme models with $B = 250$, the average regret is worse than that of the fully sequential NeuralUCB by only twenty percent, outperforming LinUCB. Furthermore, they keep the total number of policy changes limited and improve the running time of the fully sequential NeuralUCB by a factor of eight. Second, across these configurations, every adaptive batch model outperforms all fixed batch models. For example, the adaptive BNUCB with $B=40$ and $\log(q)=30$, outperforms BNUCB with $B=250$ fixed batches. This validates our theory that the minimum number of batches required for getting the optimal $\tilde{O}(\sqrt{T})$ regret is much smaller in the adaptive batch setting compared to the fixed batch setting (order of $\log T$ vs $\sqrt{T}$).


\begin{figure*}[t]
 \centering
  \subfigure[Synthetic dataset\label{fig:cosine-time}]{\includegraphics[width=0.44\textwidth]{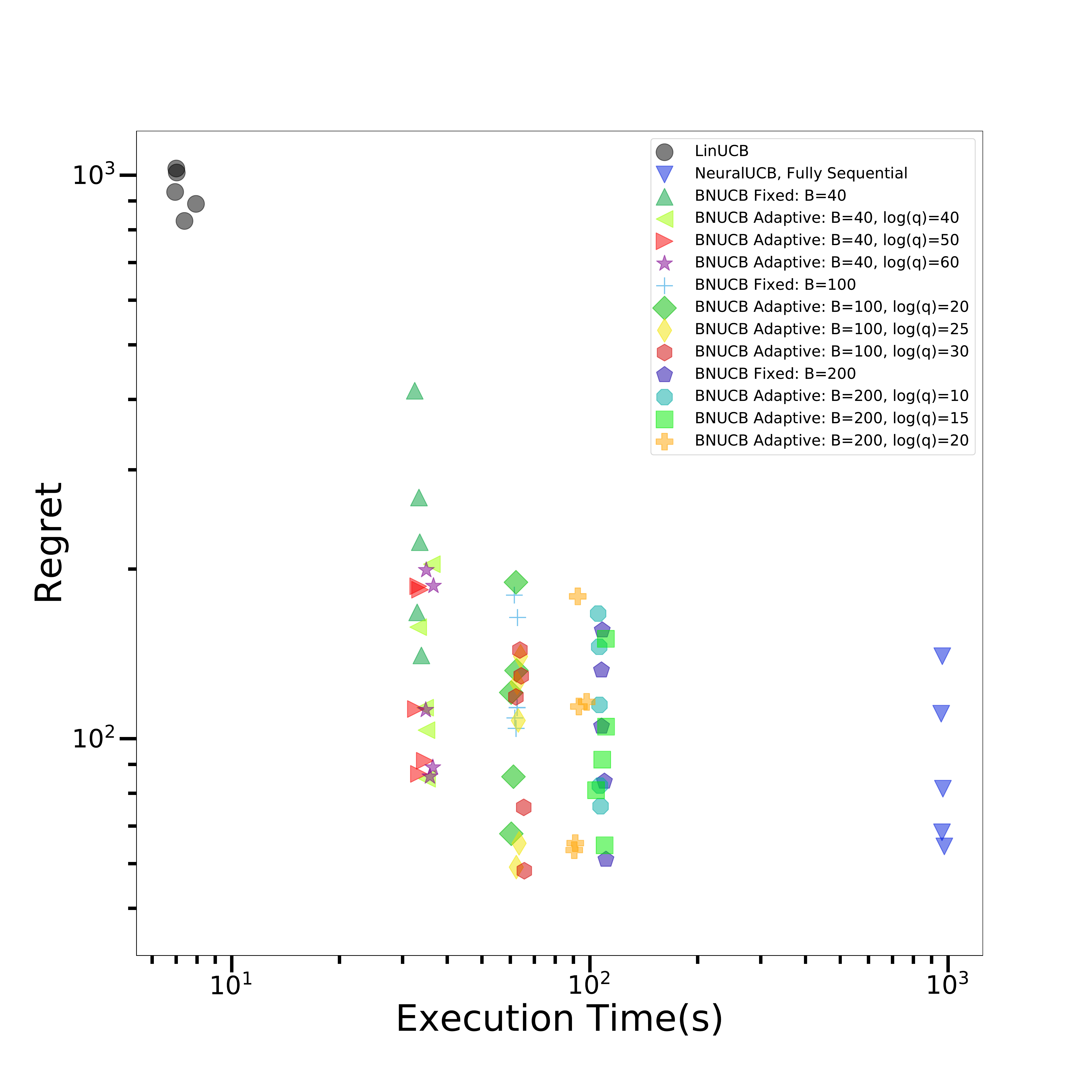}}
 \subfigure[Mushroom dataset\label{fig:mushroom-time}]{\includegraphics[width=0.44\textwidth]{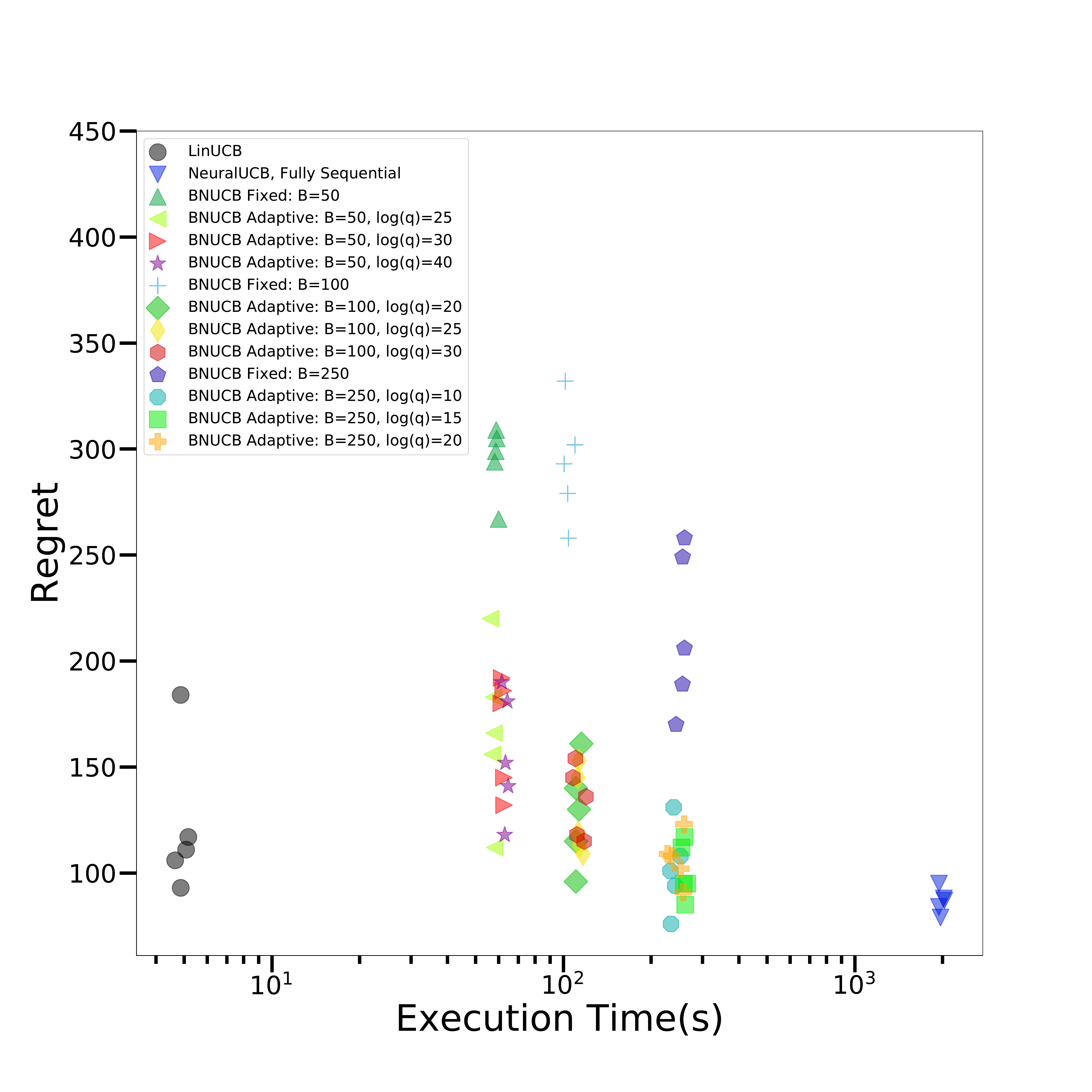}}
 \centering
   \caption{Scatter plot of regret vs execution time for $5$ instances selected at random on synthetic data (left) and Mushroom data (right).}\label{fig:simulations}
\end{figure*}

\section{Proof of the Main Results}\label{sec:proofs}
In this section, we prove Theorem \ref{thm:batch} and Theorem \ref{thm:rare}. 
\subsection{Proof of Theorem \ref{thm:batch}}
To prove Theorem \ref{thm:batch}, we need the following lemmas. The first lemma from \citet{zhou2020neural} suggests that at each time $t$ within the $b$-th batch, the difference between the reward of the optimal action $\xb_{t, a_t^*}$ and the selected action $\xb_{t, a_t}$ can be upper bounded by a bonus term defined based on the confidence radius $\beta_{t_b}$, the gradients $\gb(\xb_{t, a_t}; \btheta_b)$ and the covariance matrix $\Zb_{t_b}$. 
\begin{lemma}[Lemma 5.3, \citealt{zhou2020neural}]\label{lemma:adaptiveboundh}
Suppose Assumptions \ref{assumption:input} and \ref{assumption:context} hold. Let $a_t^* = \argmax_{a \in [K]}h(\xb_{t, a})$. There exist positive constants $\{\bar C_i\}_{i=1}^4$ such that for any $\delta \in (0,1)$, if $\eta \leq \bar C_1(mTL + m\lambda)^{-1}$ and $m = \text{poly}(T,L,K,\lambda^{-1}, \lambda_0^{-1},  S^{-1}, \log(1/\delta))$, 
 then with probability at least $1-\delta$, the following holds for all $0\leq b\leq B$ and $t_b \leq t <t_{b+1}$,
\begin{align}
    &h\big(\xb_{t,a_t^*}\big) - h\big(\xb_{t,a_t}\big)\leq 2\hat\beta_{t_b}\min\bigg\{\|\gb(\xb_{t,a_t}; \btheta_b)/\sqrt{m}\|_{\Zb_{t_b}^{-1}}, 1\bigg\} + m^{-1/6}\sqrt{\log m} \xi(T),\notag
\end{align}
where $\xi(T) = \bar C_3\big(S T^{7/6}\lambda^{-1/6}L^{7/2} +  T^{5/3}\lambda^{-2/3}L^3\big)$ and
 \begin{align}
        \hat\beta_t &=\bar C_4 \sqrt{1+m^{-1/6}\sqrt{\log m}L^4t^{7/6}\lambda^{-7/6}}  \cdot \bigg(\nu\sqrt{\log\frac{\det \Zb_t}{\det \lambda \Ib} + m^{-1/6}\sqrt{\log m} L^4t^{5/3}\lambda^{-1/6}-2\log\delta}\notag \\
        &\quad + \sqrt{\lambda}S\bigg) + (\lambda +  tL)\Big[(1- \eta m \lambda)^{J/2} \sqrt{t/\lambda}  + m^{-1/6}\sqrt{\log m}L^{7/2}t^{5/3}\lambda^{-5/3}(1+\sqrt{t/\lambda})\Big].\notag
    \end{align}
\end{lemma}

Next lemma bounds the log-determinant of covariance matrix $\Zb_{T+1}$ by the effective dimension $\tilde d$. 
\begin{lemma}\label{lemma:newzhou}
There exists a constant $\bar C >0$ such that if $m = \text{poly}(T,L,K,\lambda^{-1}, \lambda_0^{-1},  S^{-1}, \log(1/\delta))$, then with probability at least $1-\delta$, we have
\begin{align}
    \log\frac{\det \Zb_{T+1}}{\det \lambda \Ib} &\leq \tilde d\log(1+TK/\lambda) +1 + \bar C m^{-1/6}\sqrt{\log m} L^4T^{5/3}\lambda^{-1/6}.\notag
\end{align}
\end{lemma}
\begin{proof}
By Eqs. (B. 16) and (B. 19) in \citet{zhou2020neural}, we can obtain our statement when $m$ is large enough.  
\end{proof}


\begin{lemma}[Lemma 11, \citealt{abbasi2011improved}]\label{lemma:sumcontext}
For any $\{\xb_t\}_{t=1}^T \subset \RR^d$ that satisfies $\|\xb_t\|_2 \leq L$, let $\Ab_0 = \lambda \Ib$ and $\Ab_t = \Ab_0 + \sum_{i=1}^{t-1}\xb_i\xb_i^\top$, then we have
\begin{align}
    \sum_{t=1}^T \min\{1, \|\xb_t\|_{\Ab_{t}^{-1}}\}^2 \leq 2\log\frac{\det \Ab_{T+1}}{\det \lambda \Ib}.\notag
\end{align}
\end{lemma}

\begin{lemma}[Lemma 12, \citealt{abbasi2011improved}]\label{lemma:det}
Suppose $\Ab, \Bb\in \RR^{d \times d}$ are two positive definite matrices satisfying $\Ab \succeq \Bb$, then for any $\xb \in \RR^d$, $\|\xb\|_{\Ab} \leq \|\xb\|_{\Bb}\cdot \sqrt{\det(\Ab)/\det(\Bb)}$.
\end{lemma}

Now we begin our proof of Theorem \ref{thm:batch}. 
\begin{proof}[Proof of Theorem \ref{thm:batch}]
Define the set $\cC$ as follows:
\begin{align}
    \cC = \{b \in \{0,\cdots, B-1\}: \det(\Zb_{t_{b+1}})/\det(\Zb_{t_{b}})>2\}.\notag
\end{align}
Then we have for every $b \notin \cC$ and $t_{b} \leq t <t_{b+1}-1$, 
\begin{align}
    \frac{\det(\Zb_{t})}{\det(\Zb_{t_{b}})} \leq \frac{\det(\Zb_{t_{b+1}})}{\det(\Zb_{t_{b}})} \leq 2.\notag
\end{align}
Based on $\cC$, we decompose the regret as follows. 
\begin{align}
    R_T 
    &= \sum_{b \in \cC}\sum_{t = t_{b}}^{t_{b+1}-1}[h\big(\xb_{t,a_t^*}\big) - h\big(\xb_{t,a_t}\big)] + \sum_{b \notin \cC}\sum_{t = t_{b}}^{t_{b+1}-1}[h\big(\xb_{t,a_t^*}\big) - h\big(\xb_{t,a_t}\big)]\notag \\
    & \stackrel{\text{(a)}}{\leq} |\cC|\cdot T/B\cdot 2 + \sum_{b \notin \cC}\sum_{t = t_{b}}^{t_{b+1}-1} \bigg[2\hat\beta_t\min\bigg\{\|\gb(\xb_{t,a_t}; \btheta_t)/\sqrt{m}\|_{\Zb_{t_b}^{-1}}, 1\bigg\}  + 
     m^{-1/6}\sqrt{\log m} \xi(T)\bigg]\notag \\
     &\stackrel{\text{(b)}}{\leq} 2|\cC| T/B +m^{-1/6}\sqrt{\log m} \xi(T)T +  2\sqrt{2}\cdot \hat\beta_T\sum_{b \notin \cC}\sum_{t = t_{b}}^{t_{b+1}-1}\min\bigg\{\|\gb(\xb_{t,a_t}; \btheta_t)/\sqrt{m}\|_{\Zb_{t}^{-1}}, 1\bigg\},\notag
\end{align}
where (a) holds since $h\big(\xb_{t,a_t^*}\big) - h\big(\xb_{t,a_t}\big) \leq 1$ and Lemma \ref{lemma:adaptiveboundh} and (b) holds since $b \notin \cC$ and Lemma~\ref{lemma:det}. Hence, 
\begin{align}
    &R_T - 2|\cC| T/B -m^{-1/6}\sqrt{\log m} \xi(T)T \notag \\
    & \quad \stackrel{\text{(a)}}{\leq}  2\sqrt{2}\hat\beta_T\sqrt{T\sum_{t=1}^T\min\bigg\{\|\gb(\xb_{t,a_t}; \btheta_t)/\sqrt{m}\|_{\Zb_{t}^{-1}}^2, 1\bigg\}}\notag \\
    & \quad \stackrel{\text{(b)}}{\leq} 2\sqrt{2}\hat\beta_T\sqrt{T \log\frac{\det \Zb_{T+1}}{\det \lambda \Ib}}\notag,
\end{align}
where (a) holds due to Cauchy-Schwarz inequality and (b) holds due to Lemma \ref{lemma:sumcontext}. We can bound $|\cC|$ as follows:
\begin{align}
    \frac{\det(\Zb_{T+1})}{\det(\lambda \Ib)} = \prod_{b=0}^{B-1}\frac{\det(\Zb_{t_{b+1}})}{\det(\Zb_{t_{b}})} \stackrel{\text{(a)}}{\geq} \prod_{b\in \cC}\frac{\det(\Zb_{t_{b+1}})}{\det(\Zb_{t_{b}})} \stackrel{\text{(b)}}{\geq} 2^{|\cC|},\label{eq:turn1}
\end{align}
where (a) holds since $\Zb_{t_{b+1}} \succeq \Zb_{t_{b}}$ and (b) holds due to the definition of $\cC$.
Eq. \eqref{eq:turn1} suggests that $|\cC| \leq \log(\det(\Zb_{T+1})  / \det(\lambda \Ib))$. 
Therefore, 
\begin{align}
    R_T&\leq 2\log\frac{\det(\Zb_{T+1})}{\det(\lambda \Ib)} T/B +m^{-1/6}\sqrt{\log m} \xi(T)T  + 2\sqrt{2}\hat\beta_T\sqrt{T \log\frac{\det \Zb_{T+1}}{\det \lambda \Ib}}.\label{lemma:final1}
\end{align}
Finally, with a large enough $m$, by the selection of $J$ and Lemma \ref{lemma:newzhou}, we have $\log(\det \Zb_T/\log \lambda \Ib)= \tilde O(\tilde d)$ and $\hat\beta_T = \tilde O(\nu\sqrt{\log(\det \Zb_T/\log \lambda \Ib)} + \sqrt{\lambda }S) =\tilde O(\nu\sqrt{\tilde d} + \sqrt{\lambda}S)$. We also have $m^{-1/6}\sqrt{\log m} \xi(T) T \leq 1$. Substituting these terms into Eq. \eqref{lemma:final1}, we complete the proof. 
\end{proof}

\subsection{Proof of Theorem \ref{thm:rare}}
Let $B'$ be the value of $b$ when Algorithm \ref{algorithm:3} stops. It is easy to see that $B' \leq B$, therefore there are at most $B$ batches. We can first decompose the regret as follows, using Lemma \ref{lemma:adaptiveboundh}:
\begin{align}
    R_T = \sum_{t=1}^T [h\big(\xb_{t,a_t^*}\big) - h\big(\xb_{t,a_t}\big)]
    \leq  2 \sum_{b=0}^{B'}\sum_{t=t_b}^{t_{b+1}-1}\hat\beta_t\min\bigg\{\|\gb(\xb_{t,a_t}; \btheta_{t})/\sqrt{m}\|_{\Zb_{t_b}^{-1}}, 1\bigg\}
    + m^{-1/6}\sqrt{\log m} \xi(T) T,\label{eq:startregret}
\end{align}
To bound Eq. \eqref{eq:startregret}, we have the following two separate cases. First, if $B'<B$, then for all $0 \leq b \leq B'$ and $t_b \leq t<t_{b+1}$, we have $\det(\Zb_t) \leq q \det(\Zb_{t_b})$. Therefore, we have
\begin{align}
R_T - m^{-1/6}\sqrt{\log m} \xi(T) T\notag 
    &\quad \stackrel{\text{(a)}}{\leq}  2\hat\beta_T\sqrt{q} \sum_{b=0}^{B'}\sum_{t=t_b}^{t_{b+1}-1}\min\bigg\{\|\gb(\xb_{t,a_t}; \btheta_{t})/\sqrt{m}\|_{\Zb_{t}^{-1}}, 1\bigg\} \notag\\
    & \quad\stackrel{\text{(b)}}{\leq}  2\hat\beta_T\sqrt{q} \sqrt{T}\sqrt{\sum_{t=1}^T \min \bigg\{\|\gb(\xb_{t,a_t}; \btheta_{t})/\sqrt{m}\|_{\Zb_{t}^{-1}}^2, 1\bigg\}}\notag \\
    & \quad\stackrel{\text{(c)}}{\leq}  2\hat\beta_T\sqrt{T}\sqrt{q}\sqrt{2 \log\frac{\det \Zb_T}{\det \lambda \Ib}},\label{eq:finalregret_1}
\end{align}
where (a) holds due to Lemma \ref{lemma:det}, (b) holds due to Cauchy-Schwarz inequality and (c) holds due to Lemma \ref{lemma:sumcontext}. Second, if $B' = B$, then for all $0 \leq b \leq B-1$, we have $\det(\Zb_{t_{b+1}}) > q \det(\Zb_{t_b})$. For $b = B$ and $t_B \leq t<t_{B+1}$, we have
\begin{align}
    \frac{\det(\Zb_t)}{\det(\Zb_{t_B})} \leq \frac{\det(\Zb_T)}{\det(\lambda \Ib)}\cdot \prod_{b=0}^{B-1}\frac{\det(\Zb_{t_b})}{\det(\Zb_{t_{b+1}})} \leq \frac{\det(\Zb_T)}{\det(\lambda \Ib)}\cdot q^{-B}. \label{eq:detbound}
\end{align}
Therefore, by Eq. \eqref{eq:startregret} the regret can be bounded as
\begin{align}
&R_T - m^{-1/6}\sqrt{\log m} \xi(T) T\notag \\
    &\leq  2 \sum_{b=0}^{B-1}\sum_{t=t_b}^{t_{b+1}-1}\hat\beta_t\min\bigg\{\|\gb(\xb_{t,a_t}; \btheta_{t})/\sqrt{m}\|_{\Zb_{t_b}^{-1}}, 1\bigg\} \notag \\
    &\qquad + \sum_{t=t_B}^{t_{B+1}-1}\hat\beta_t\min\bigg\{\|\gb(\xb_{t,a_t}; \btheta_{t})/\sqrt{m}\|_{\Zb_{t_b}^{-1}}, 1\bigg\}\notag \\
    &\stackrel{\text{(a)}}{\leq} 2\hat\beta_T\sqrt{q} \sum_{b=0}^{B-1}\sum_{t=t_b}^{t_{b+1}-1}\min\bigg\{\|\gb(\xb_{t,a_t}; \btheta_{t})/\sqrt{m}\|_{\Zb_{t}^{-1}}, 1\bigg\} \notag \\
    &\qquad + \hat\beta_T\sqrt{\frac{\det(\Zb_T)}{q^B\det(\lambda \Ib)}} \cdot \sum_{t=t_B}^{t_{B+1}-1}\min\bigg\{\|\gb(\xb_{t,a_t}; \btheta_{t})/\sqrt{m}\|_{\Zb_{t}^{-1}}, 1\bigg\}\notag \\
    & \leq 2\hat\beta_T\max\bigg\{\sqrt{q},\sqrt{\frac{\det(\Zb_T)}{q^B\det(\lambda \Ib)}} \bigg\} \cdot\sum_{t=1}^T\min\bigg\{\|\gb(\xb_{t,a_t}; \btheta_{t})/\sqrt{m}\|_{\Zb_{t}^{-1}}, 1\bigg\}\notag\\
    & \stackrel{\text{(b)}}{\leq}  2\hat\beta_T\max\bigg\{\sqrt{q},\sqrt{\frac{\det(\Zb_T)}{q^B\det(\lambda \Ib)}} \bigg\} \sqrt{T}\cdot \sqrt{\sum_{t=1}^T \min \bigg\{\|\gb(\xb_{t,a_t}; \btheta_{t})/\sqrt{m}\|_{\Zb_{t}^{-1}}^2, 1\bigg\}} \notag \\
    & \stackrel{\text{(c)}} {\leq}   2\hat\beta_T\sqrt{T}\max\bigg\{\sqrt{q},\sqrt{\frac{\det(\Zb_T)}{q^B\det(\lambda \Ib)}} \bigg\}\sqrt{2 \log\frac{\det \Zb_T}{\det \lambda \Ib}},\label{eq:finalregret}
\end{align}
where (a) holds due to Lemma \ref{lemma:det} and the following two facts: $\det(\Zb_t) \leq q \det(\Zb_{t_b})$ for all $0 \leq b \leq B-1$, $t_b \leq t<t_{b+1}$; Eq. \eqref{eq:detbound} for $b = B$, (b) holds due to Cauchy-Schwarz inequality and (c) holds due to Lemma \ref{lemma:sumcontext}. Combining Eqs. \eqref{eq:finalregret_1} and \eqref{eq:finalregret}, we have under both $B'<B$ and $B' = B$ cases, Eq. \eqref{eq:finalregret} holds. Finally, by Lemma \ref{lemma:newzhou} and the selection of $J$ and $m$, we have
\begin{align}
    &\frac{\det \Zb_T}{\log \lambda \Ib} = \tilde O((1+TK/\lambda)^{\tilde d}),\notag \\
    &\log\frac{\det \Zb_T}{\log \lambda \Ib} = \tilde O(\tilde d),\notag \\
    &\hat\beta_T  = \tilde O\bigg(\nu\sqrt{\log\frac{\det \Zb_T}{\log \lambda \Ib}} + \sqrt{\lambda}S\bigg) = \tilde O\Big(\nu\sqrt{\tilde d} + \sqrt{\lambda}S\Big).\label{eq:sub1}
\end{align}
Thus, substituting Eq. \eqref{eq:sub1} into Eq. \eqref{eq:finalregret}, we have
\begin{align}
    R_T &= \tilde O\bigg(\Big(\nu\sqrt{\tilde d} + \sqrt{\lambda}S\Big)\cdot\sqrt{T} \cdot \max\Big\{\sqrt{q}, \sqrt{(1+TK/\lambda)^{\tilde d}/q^B}\Big\} \cdot \sqrt{\tilde d}\bigg)\notag \\
    & = O\bigg(\sqrt{\max\{q, (1+TK/\lambda)^{\tilde d}/q^B\}}\Big(\nu\tilde d + \sqrt{\lambda\tilde d}S\Big)\sqrt{T}  \bigg).\notag
\end{align}

\section{Conclusions}

In this paper, we proposed the BatchNeuralUCB algorithm which combines neural networks with the UCB technique to balance exploration-exploitation tradeoff while keeping the total number of batches limited. We studied both fixed and adaptive batch size settings and proved that BatchNeuralUCB achieves the same regret as the fully sequential version. Our theoretical results are complemented by experiments on both synthetic and real-world datasets.

\appendix

\vskip 0.3in


\section{More on Experiments}\label{app:additional_sim}

\subsection{Details of Experiments in Section \ref{sec:sim}}
For experiment in Section \ref{sec:sim-synth}, we select parameters as follows. For LinUCB, we again search over the regularization parameter $\lambda \in \{0.001, 0.01, 0.1, 1\}$ and exploration parameter $\beta \in \{0.001, 0.01, 0.1, 1\}$ and pick the best model. For NeuralUCB and $\algname$, we train two-layers neural networks with $m=200$ hidden layers. For both of these algorithms, we find that choosing parameters $\lambda = 0.01$ and $\beta_t = 0.001$ works pretty well. Finally, in the iterations where the policy is updated, the parameters of neural networks are updated for $J=200$ iterations using stochastic gradient descent with $\eta = 0.01$.

For experiment in Section \ref{sec:sim-real}, we select the parameters as follows. For LinUCB, we search over the space of regularization parameters $\lambda \in \{0.001, 0.01, 0.1, 1\}$ and the exploration parameter $\beta \in \{0.001, 0.01, 0.1, 1\}$ and report the model with the lowest average regret. For BatchNeuralUCB and NeuralUCB algorithms, we consider two-layer neural networks with $m=100$ hidden layers. For both algorithms, during iterations that policy update is allowed (at the end of batches for BatchNeuralUCB and every iteration for NeuralUCB), we use stochastic gradient descent with $J=200$ iterations and $\eta = 0.05$ to update the network parameters. For both of these algorithms, we find that choosing parameters $\lambda = 0.001$ and $\beta_t = 0.001$ works well. 

\subsection{Additional Experiments}

We repeat our numerical experiments in Section \ref{sec:sim}, with one additional synthetic dataset as well as one additional real dataset. 

\subsection*{Synthetic Data with Quadratic Reward}

\begin{figure*}[hptb]
    \centering
    \includegraphics[width=0.9\textwidth]{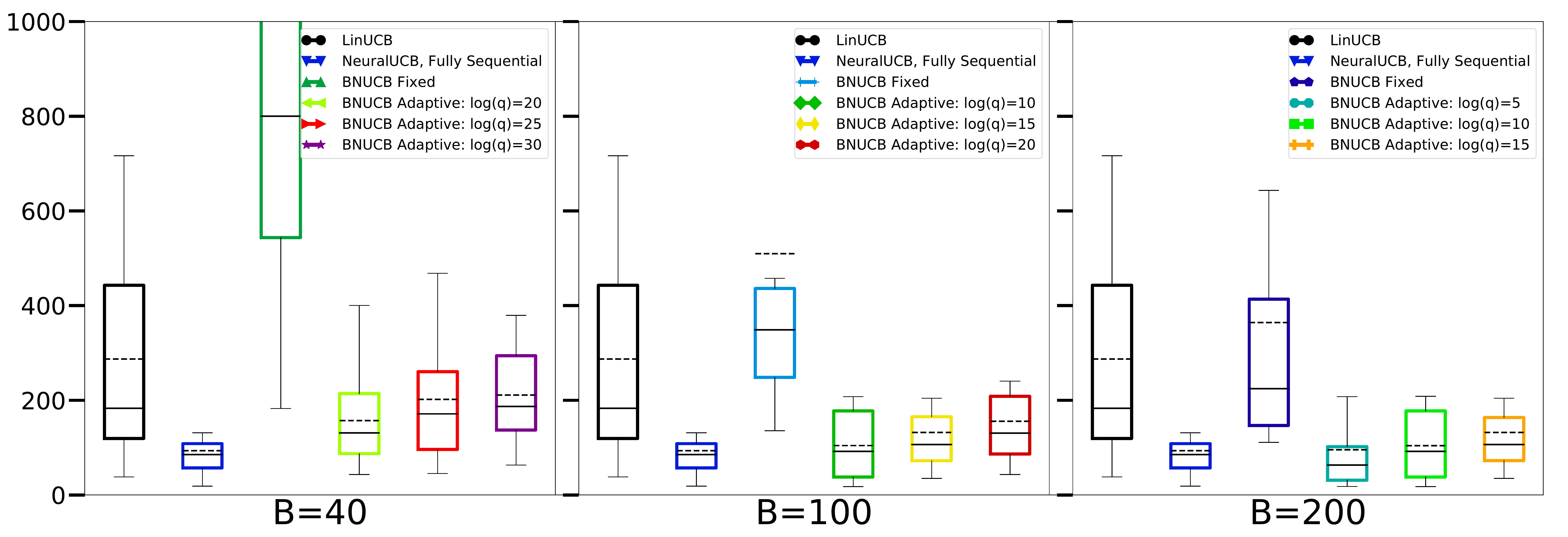}
    \caption{Distribution of per-instance regret on Synthetic data with quadratic reward. The solid and dashed lines indicate the median and the mean respectively. Note that BNUCB stands for $\algname$.} \label{fig:quadratic}
\end{figure*}

\begin{figure*}[hptb]
    \centering
    \includegraphics[width=0.75\textwidth]{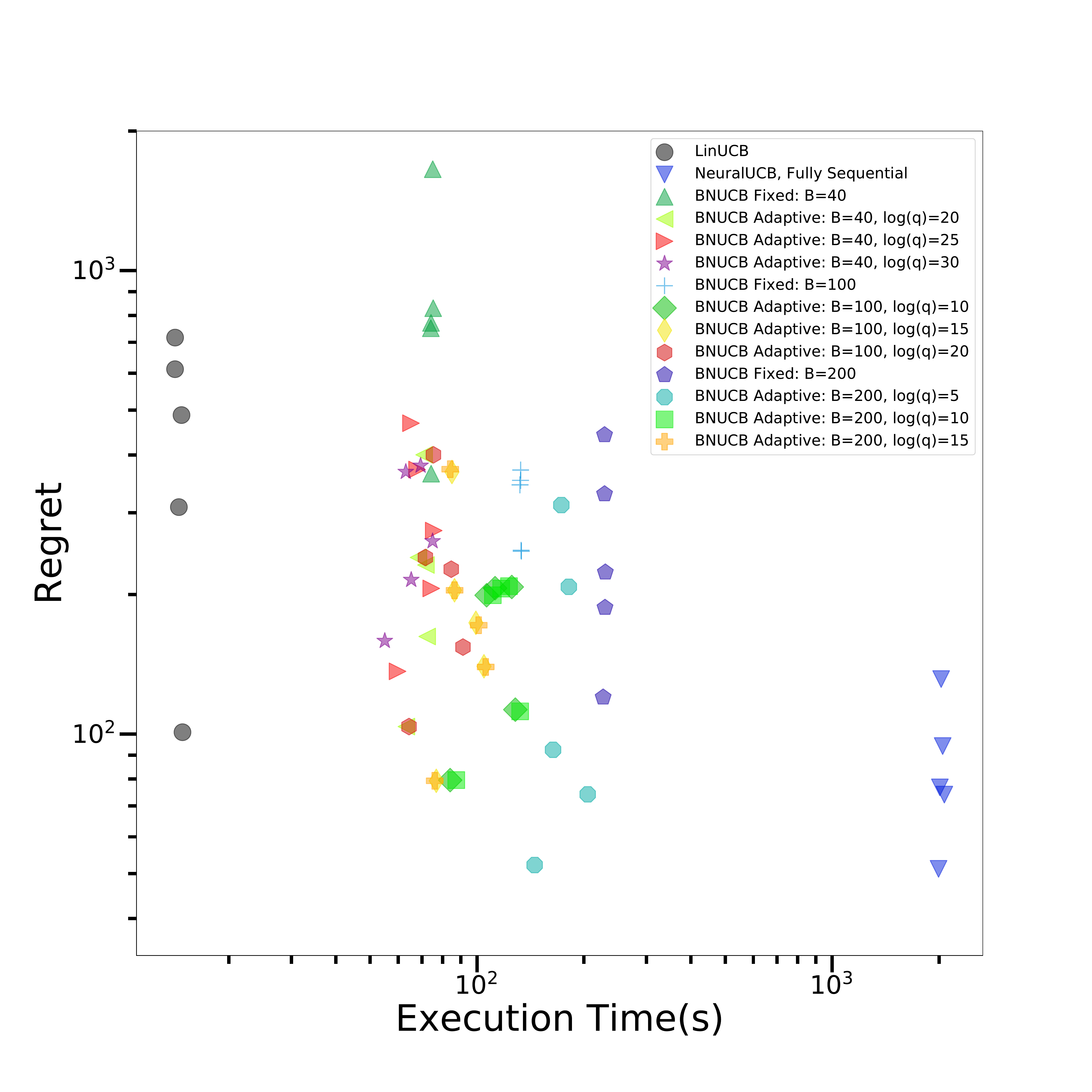}
    \caption{Scatter plot of regret vs execution time for $5$ instances selected at random on synthetic data with quadratic reward.} \label{fig:quadratic-time}
\end{figure*}

We compare the performance of our proposed $\algname$ (BNUCB) algorithm with fixed and adaptive size batches for several values of $B$ and $q$ with two fully sequential benchmarks on synthetic data generated as follows. Let $T=2000, d=4, K=10$, and consider the quadratic reward function given by $r_{t,a} = \xb_{t,a}^\top \Ab^\top \Ab \xb_{t,a} + \xi_t$ where $\{\xb_{t,1}, \xb_{t,2}, \cdots, \xb_{t,K}\}$ are contexts generated at time $t$ according to $U[0, 1]^d$ independent of each other. Each entry of the matrix $\Ab \in \RR^{d \times d}$ is generated according to $N(0, 1)$. The noise $\xi_t$ is generated according to $N(0, 0.25)$ independent of all other variables. 

We consider the same benchmarks as those presented in Section \ref{sec:sim}. Similar to Section \ref{sec:sim}, we repeat our simulations for $10$ times and plot the following charts: (1) the box plot of total regret of algorithms with its standard deviation, and (2) scatter plot of the total regret vs execution time for 5 simulations randomly selected out of all $10$ simulations.

We select parameters as follows. For LinUCB, we again search over the regularization parameter $\lambda \in \{0.001, 0.01, 0.1, 1\}$ and exploration parameter $\beta \in \{0.001, 0.01, 0.1, 1\}$ and pick the best model. For NeuralUCB and $\algname$, we train two-layers neural networks with $m=100$ hidden layers. For both of these algorithms, we find that choosing parameters $\lambda = 0.01$ and $\beta_t = 0.01$ works pretty well. Finally, in the iterations where the policy is updated, the parameters of neural networks are updated for $J=200$ iterations using stochastic gradient descent with $\eta = 0.005$.

\paragraph{Results.} The results are depicted in Figures \ref{fig:quadratic} and \ref{fig:quadratic-time}. As can be observed, due to model misspecifications, LinUCB does not perform very well. On the other hand, the fully sequential NeuralUCB algorithm outperforms all other algorithms. However, this algorithm is computationally very expensive and it requires almost $2000$ seconds per execution. The proposed BatchNeuralUCB algorithms with both fixed and adaptive batch sizes performs very well. In particular, the adaptive BatchNeuralUCB algorithm with only $B=40$ batches and all configurations for $q$, i.e. $\log(q) = 20, 25, 30$, outperforms LinUCB and achieves a very close performance to that of NeuralUCB while enjoying a very fast execution time of around $90$ seconds on average. The gap in the regret with the fully sequential NeuralUCB algorithm becomes smaller for configurations with $B=100$ and it becomes almost insignificant for $B=200$.

\subsection*{Real Magic Telescope Data}
\begin{figure*}[hptb]
    \centering
    \includegraphics[width=0.9\textwidth]{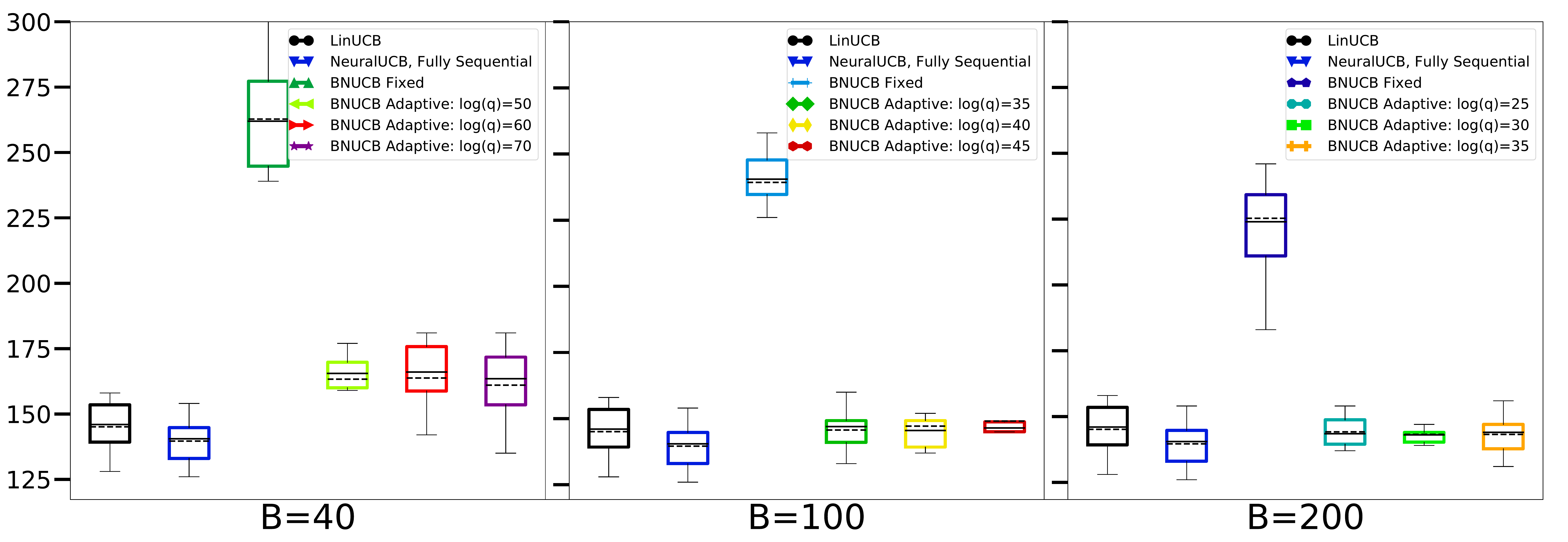}
    \caption{Distribution of per-instance regret on Magic Telescope dataset. The solid and dashed lines indicate the median and the mean respectively. Note that BNUCB stands for $\algname$.} \label{fig:magic}
\end{figure*}
\begin{figure*}[ht]
    \centering
    \includegraphics[width=0.75\textwidth]{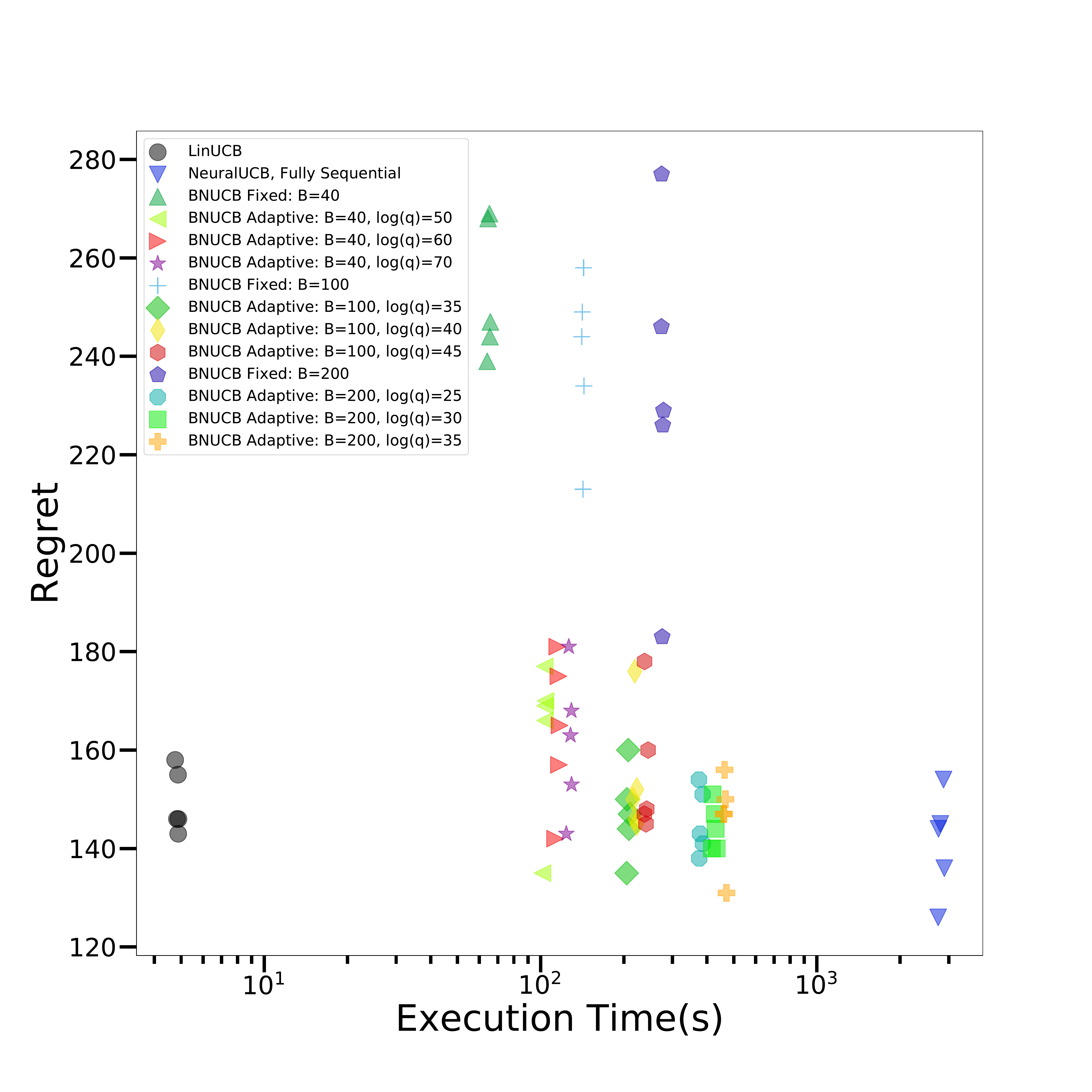}
    \caption{Scatter plot of regret vs execution time for $5$ instances selected at random on Magic Telescope dataset.} \label{fig:magic-time}
\end{figure*}
We repeat the simulation in Section \ref{sec:sim-real} this time using the MAGIC Gamma Telescope dataset from the UCI repository.\footnote{This dataset can be found here \href{http://archive.ics.uci.edu/ml/datasets/MAGIC+GAMMA+Telescope}{http://archive.ics.uci.edu/ml/datasets/MAGIC+GAMMA+Telescope}} The dataset is originally designed for classifying gamma vs hadron meteor showers. It contains $n=19020$ samples each with $d=10$ features, each belonging to one of $K=2$ classes. For each sample $s_t = (c_t, l_t)$ with context $c_t \in \RR^d$ and label $l_t \in [K]$, we consider the zero-one reward defined as $r_{t, a_t} = \ind\{a_t = l_t\}$ and generate our context vectors as $\{\xb_{t,a} \in \RR^{Kd}: \xb_{t,a} = [\underbrace{0, \cdots, 0}_{a-1 \text{~times}}, c_i, \underbrace{0, \cdots, 0}_{K-a \text{~times}}], a \in [K]\}$. 

We compare the performance of algorithms on similar metrics as those described in Section~\ref{sec:sim-real} over $10$ random Monte Carlo simulations and report the results. For each instance, we select $T=2000$ random samples without replacement from the dataset and run all algorithms on that instance. Note that in this simulation, both NeuralUCB and $\algname$ use two-layer neural networks with $m=100$ hidden layers. 

We select parameters as follows. For LinUCB, we again search over the regularization parameter $\lambda \in \{0.001, 0.01, 0.1, 1\}$ and exploration parameter $\beta \in \{0.001, 0.01, 0.1, 1\}$ and pick the best model. For NeuralUCB and $\algname$, we train two-layers neural networks with $m=400$ hidden layers. For both of these algorithms, we find that choosing parameters $\lambda = 0.001$ and $\beta_t = 0.001$ works pretty well. Finally, in the iterations where the policy is updated, the parameters of neural networks are updated for $J=200$ iterations using stochastic gradient descent with $\eta = 0.05$.

\paragraph{Results.} The results are depicted in Figures \ref{fig:magic} and \ref{fig:magic-time}. As can be observed, both fully sequential algorithms, i.e. LinUCB and NeuralUCB perform relatively well with NeuralUCB outperforming LinUCB by a slight margin. Across batch models, adaptive ones perform much better than fixed ones. For instance, BatchNeuralUCB with $B=40$ adaptively chosen using $\log(q) = 50, 60,$ or $70$ outperforms BatchNeuralUCB with $B=200$ batches of fixed size. The average regret of both fixed and adaptive versions reduces as the number of batches increases. In particular, the regret of adaptive models with $B=100$ and $B=200$ are very close and indistinguishable from the fully sequential NeuralUCB algorithm, while taking almost $8$ times less time for execution.

\bibliographystyle{ims}
\bibliography{reference}
\end{document}